\documentclass{article}

\usepackage{microtype}
\usepackage{graphicx}
\usepackage{subcaption}
\usepackage{booktabs} 

\usepackage{hyperref}

\usepackage[accepted]{arxiv2025}

\usepackage{amsmath}
\usepackage{amssymb}
\usepackage{mathtools}
\usepackage{amsthm}

\usepackage[capitalize,noabbrev]{cleveref}
\crefname{ALC@line}{line}{lines}

\theoremstyle{plain}
\newtheorem{theorem}{Theorem}[section]
\newtheorem*{theorem*}{Theorem}

\newtheorem{lemma}[theorem]{Lemma}
\newtheorem*{lemma*}{Lemma}
\newtheorem{corollary}[theorem]{Corollary}
\theoremstyle{definition}
\newtheorem{definition}[theorem]{Definition}
\newtheorem{assumption}[theorem]{Assumption}
\theoremstyle{remark}

\usepackage{tikz}
\usepackage{tikzsymbols}
\usetikzlibrary{arrows}
\usepackage{tabularx}
\usepackage{multirow}
\usepackage{array}
\usepackage{makecell}
\usepackage{arydshln}
\usepackage{pifont}
\usepackage{bbding}
\usepackage{stmaryrd}
\usepackage{blindtext}
\usepackage{mathrsfs}
\usepackage[export]{adjustbox}
\usepackage{thmtools}
\usepackage{thm-restate}
\usepackage[colorinlistoftodos,textwidth=2.1cm]{todonotes}

\definecolor{tealblue}{HTML}{66C7C7}

\newcommand{\alglinelabel}[1]{%
  \addtocounter{ALC@line}{-1}%
  \refstepcounter{ALC@line}%
  \label{#1}%
}

\newcommand{\lemref}[1]{Lemma~(\ref{#1})}
\newcommand{\theoref}[1]{Theorem~(\ref{#1})}

\newcommand{\eg}[0]{\textit{e.g., }}

\newcommand{\ie}[0]{\textit{i.e., }}

\newcommand{\kdp}[0]{\kappa_{\eps,\delta}}
\newcommand{\abs}[1]{\left| #1 \right|}

\newcommand{\D}[0]{\Dcal}
\newcommand{\Df}[0]{\Dcal_f}
\newcommand{\Dr}[0]{\Dcal_r}

\newcommand{\thetilde}[0]{\widetilde{\vtheta}}
\newcommand{\kl}[0]{\kappa_{\ell}}

\definecolor{darkgreen}{RGB}{0,180,0}
\definecolor{purple}{RGB}{200,0, 200}
\definecolor{royal_blue}{RGB}{90, 120, 250}

\definecolor{tealblue}{HTML}{66C7C7}
\definecolor{coolgreen}{RGB}{0, 155, 119}

\usepackage{xspace}
\newcommand{\VRU}{\hyperref[alg:VRU]{\textit{VRU}}\xspace}

\newcommand{\norm}[1]{\left\lVert#1\right\rVert}

\newcommand{\Acal}{\mathcal{A}}

\newcommand{\Dcal}{\mathcal{D}}

\newcommand{\Fcal}{\mathcal{F}}

\newcommand{\Lcal}{\mathcal{L}}

\newcommand{\Ncal}{\mathcal{N}}
\newcommand{\Ocal}{\mathcal{O}}
\newcommand{\Pcal}{\mathcal{P}}

\newcommand{\Ucal}{\mathcal{U}}

\newcommand{\expect}{\operatorname{\mathbb{E}}}

\newcommand{\E}[1]{\expect\left[#1\right]}

\usepackage{amsmath,amsfonts,bm}

\def\Figref#1{Figure~\ref{#1}}

\def\secref#1{section~\ref{#1}}

\def\eqref#1{equation~\ref{#1}}
\def\Eqref#1{Eq.~\ref{#1}}

\def\Algref#1{Alg.~\ref{#1}}

\def\1{\bm{1}}

\def\eps{{\epsilon}}

\def\vtheta{{\bm{\theta}}}

\def\vx{{\bm{x}}}
\def\vy{{\bm{y}}}

\DeclareMathAlphabet{\mathsfit}{\encodingdefault}{\sfdefault}{m}{sl}
\SetMathAlphabet{\mathsfit}{bold}{\encodingdefault}{\sfdefault}{bx}{n}

\def\sA{{\mathbb{A}}}

\def\sE{{\mathbb{E}}}

\def\sN{{\mathbb{N}}}

\def\sU{{\mathbb{U}}}

\newcommand{\nablaT}{\widetilde{\nabla}}

\icmltitlerunning{VRU: Variance-Reduced Unlearning}
\begin{document}

\twocolumn[
\icmltitle{Variance-Reduced $(\varepsilon, \delta)-$Unlearning using Forget Set Gradients}

\begin{icmlauthorlist}
\icmlauthor{Martin Van Waerebeke}{yyy}
\icmlauthor{Marco Lorenzi}{zzz}
\icmlauthor{Kevin Scaman}{yyy}
\icmlauthor{El Mahdi El Mhamdi}{mmm}
\icmlauthor{Giovanni Neglia}{zzz}
\end{icmlauthorlist}

\icmlaffiliation{yyy}{INRIA Paris}
\icmlaffiliation{zzz}{INRIA Sophia Antipolis}
\icmlaffiliation{mmm}{CMAP - Ecole polytechnique}

\icmlcorrespondingauthor{Martin Van Waerebeke}{martin.van-waerebeke@inria.fr}

\icmlkeywords{Machine Learning, ICML}

\vskip 0.3in
]

\printAffiliationsAndNotice{}

\begin{abstract}

In machine unlearning, $(\varepsilon,\delta)-$unlearning is a popular framework that provides formal guarantees on the effectiveness of the removal of a subset of training data, the \emph{forget set}, from a trained model. 
For strongly convex objectives, existing first-order methods achieve $(\varepsilon,\delta)-$unlearning, but they only use the forget set to calibrate injected noise, never as a direct optimization signal. 
In contrast, efficient empirical heuristics often exploit the forget samples (e.g., via gradient ascent) but come with no formal unlearning guarantees.
We bridge this gap by presenting the Variance-Reduced Unlearning (\VRU) algorithm. To the best of our knowledge, \VRU is the first first-order algorithm that directly includes forget set gradients in its update rule, while provably satisfying $(\varepsilon,\delta)-$unlearning.
We establish the  convergence of \VRU and show that incorporating the forget set yields strictly improved rates, \ie a better dependence on the achieved error
compared to existing first-order $(\varepsilon,\delta)-$unlearning methods. Moreover, we prove that, in a low-error regime,  \VRU 
asymptotically outperforms
any first-order method that ignores the forget set.
Experiments corroborate our theory, showing consistent gains over both state-of-the-art certified unlearning methods and over empirical
baselines that explicitly leverage the forget set.

\end{abstract}

\section{Introduction}

Machine Unlearning (MU) aims at efficiently removing the influence of a subset of training examples (the \emph{forget set}) from a trained model so that, after unlearning, the model is equivalent to one obtained by retraining   
from scratch on the remaining data (the \emph{retain set}). 
The objective of MU is thus threefold: i) ensuring erasure of the forget set, ii) preserving model utility on the retain set, and iii) achieving substantial computational savings with respect to retraining from scratch. 

To formalize the notion of ``erasure'', the community has developed the framework of $(\varepsilon,\delta)-$unlearning \cite{ginart2019making}, which draws analogies from differential privacy \cite{DP_book}. More precisely, $(\varepsilon,\delta)-$unlearning provides statistical indistinguishability between the distributions of models obtained with a given unlearning procedure, and the one obtained by retraining on the retain set. 
Within this setting, first order methods based on variants of gradient descent are popular unlearning approaches, as their computational complexity scales well with the model size~\cite{allouah2024fast, sepahvand2025leveraging,koloskova2025certified}.  
These approaches rely primarily on fine-tuning on the retain set, and on suitable noise injection to guarantee the erasure of the forget set \cite{DescentToDelete}. 

Nevertheless, current theories for $(\varepsilon,\delta)-$unlearning present relevant \emph{theoretical} and \emph{practical} limitations, as they make limited use of the forget set during the unlearning procedure, using it primarily to calibrate the injected noise rather than as an optimization signal.

From the theoretical perspective, 
it has been recently shown that the \emph{gradient on the forget set} drives the steepest descent direction for the unlearning loss \cite{huang2024unified}. Moreover, first order methods not relying on the forget set are less efficient than retraining from scratch in the low error regime \cite{van2025forget}. From the practical standpoint, several MU methods not backed by $(\varepsilon,\delta)$ guarantees do exploit the forget set for the unlearning optimization strategy, for example by performing gradient ascent steps on the forget set, or by fine-tuning on random labels \cite{fan2023salun}. While these methods are empirical and not backed by formal guarantees, they have been shown to outperform certified $(\varepsilon,\delta)-$unlearning methods in popular benchmarks \cite{maini2024tofu, shi2024muse, li2024wmdp}. 

\textbf{Contributions.} In this work we reconcile the notion of first order $(\varepsilon,\delta)-$unlearning with the effective use of the forget set during the unlearning procedure. To this end, we extend the family of first order $(\varepsilon,\delta)-$unlearning approaches with a novel method, \textit{V}ariance-\textit{R}educed \textit{U}nlearning (\VRU), and demonstrate its superiority over the state-of-the-art from both the theoretical and practical aspects. 

In particular, we:

\begin{itemize}
    \item Propose \VRU, the first first-order $(\varepsilon,\delta)-$unlearning method that includes gradient ascent on the forget set
    to improve unlearning efficiency.
    \item Analyze \VRU for strongly-convex losses, establishing convergence rates that improve over existing first-order $(\varepsilon,\delta)-$unlearning approaches.
    \item Prove that \VRU asymptotically beats any first-order $(\varepsilon,\delta)-$unlearning algorithms that does not use the forget-set.
    \item  Practically validate the effectiveness of \VRU  against both $(\varepsilon,\delta)-$unlearning and empirical methods.
\end{itemize}

\section{Related Works}
The domain of machine unlearning (MU) studies how to efficiently remove the influence of a designated subset of training examples from a learned model, with the goal of matching the model that would be obtained by retraining from scratch on the remaining data, at a substantially lower computational cost. 

Much progress has been made recently, both towards new unlearning methods \cite{guo2024mechanistic, schoepf2025redirection} and new ways to evaluate them \cite{maini2024tofu, hayes2025inexact, lu2025waterdrum}.

\textit{Unlearning methods are commonly categorized as either exact or approximate.} The former return a model identical to the one obtained  by retraining from scratch on the retain set, whereas the latter aim to closely approximate the retrained solution.

Exact methods offer ideal privacy guarantees, but usually require a modification of the training process, usually through tree-based approaches \cite{ullah2021machine, ullah2023adaptive} or sharding \cite{sisa, yan2022arcane, wang2023fedcsa}, limiting their scope.

\textit{Approximate MU algorithms can be divided into empirical and certified methods.} Empirical approaches prioritize  efficiency and broad applicability, but do not provide  explicit unlearning guarantees, whereas certified methods provably satisfy a formal definition of unlearning.
Several notions of  certified unlearning have been proposed. Early work focused on minimizing the KL-divergence with respect to a retrained model \cite{golatkar2020forgetting, Golatkar_2020_CVPR, Golatkar_2021_CVPR,  jin2023ntk}. More recent definitions consider alternatives such as cosine similarity with retraining \cite{melamed2025provable} or more sample-specific criteria~\cite{sepahvand2025leveraging}.

However, most certified unlearning methods are based on the notion of $(\varepsilon, \delta)-$unlearning \cite{ginart2019making},  aiming at achieving statistical indistinguishability between the post-unlearning model and a model retrained from scratch. To meet this guarantee, existing algorithms rely on either first-order \cite{DescentToDelete} or second-order \cite{guo2020certified} updates. Second-order methods exploit information about the local curvature of the loss function, while first-order methods typically scale better to high-dimensional models.

\textit{First-order MU methods that achieve $(\varepsilon, \delta)$-unlearning rely on a combination of noise injection and gradient-based fine-tuning on the retain set.}
Recent work has established convergence and utility guarantees for such methods under various assumptions~\cite{DescentToDelete, fraboni2024sifu, chien2024langevin, van2025forget, koloskova2025certified, mu2024rewind}. These approaches differ primarily on how they combine the two core operations: noise injection and fine-tuning on the retain set. Existing strategies include (i) fine-tuning on the retain set followed by output perturbation~\cite{DescentToDelete, allouah2024utility}, (ii) adding noise to the model and then fine-tuning~\cite{fraboni2024sifu, van2025forget, mu2024rewind}, and (iii) injecting noise at each gradient step during fine-tuning~\cite{chien2024langevin, koloskova2025certified, sepahvand2025leveraging}.

\textit{The remaining literature rely on the forget set to achieve approximate unlearning.} In particular, most empirical methods and most second-order certified approaches exploit the forget set to their advantage.
Empirical methods often apply gradient ascent on the forget set  \cite{chen2023unlearn, yao2024large,pang2025label}, but they also leverage it in other ways: discouraging the model from reproducing the original predictions on the forget set~\cite{kurmanji2024towards_unbounded}, training on fake or homogeneous labels \cite{fan2023salun}, replacing sensitive forget-set information with generic one \cite{zhang2024negative} or adversarial \cite{yao2024machine} content.

Second-order methods typically combine the forget-set gradient with the inverse Hessian to form a Newton-like unlearning update \cite{guo2020certified, sekhari2021remember, Golatkar_2020_CVPR}. Recent work accelerates these updates by avoiding explicit inverse-Hessian computation in an online-learning setting~\cite{ qiao2024hessian}.

\textit{Forget-set ascent steps are crucial for designing efficient MU procedures.} Recent articles have provided theoretical evidence that first-order MU algorithms must incorporate forget-set–based updates to achieve efficient unlearning. Under a local second-order approximation,~\citet{huang2024unified} decompose the optimal unlearning update, proving that it must include a weighted gradient ascent component on the forget set, alongside the standard descent on the retain set.

\section{Problem Statement}
In this section, we introduce key notations for our analysis.

We consider a supervised learning setting with model parameters $\bm{\theta} \in \mathbb{R}^d$ and loss $\ell: \mathbb{R}^d \times \mathbb{R}^s \to \mathbb{R} \cup \{+\infty\}$. 
Let $\mathcal{D}_f$ and $\mathcal{D}_r$ denote the forget and retain data distributions over $\mathbb{R}^s$, respectively. The original training distribution is the mixture $\mathcal{D} \coloneqq r_f \mathcal{D}_f + (1 - r_f) \mathcal{D}_r$, where $r_f \in (0,1)$ represents the fraction of data to be unlearned. For any distribution $\mathcal{D}$, we define the population risk as
\begin{equation}
    \mathcal{L}(\bm{\theta}, \mathcal{D}) \coloneqq \mathbb{E}_{\xi \sim \mathcal{D}}[\ell(\bm{\theta}, \xi)]\:.
\end{equation}

We recall that a differentiable function $f: \mathbb{R}^d \to \mathbb{R}$ is \emph{$\mu$-strongly convex} if $f(\bm{\theta}') \geq f(\bm{\theta}) + \langle \nabla f(\bm{\theta}), \bm{\theta}' - \bm{\theta} \rangle + \frac{\mu}{2} \|\bm{\theta}' - \bm{\theta}\|^2$ for all $\bm{\theta}, \bm{\theta}'$; \emph{$\beta$-smooth} if $\|\nabla f(\bm{\theta}) - \nabla f(\bm{\theta}')\| \leq \beta \|\bm{\theta} - \bm{\theta}'\|$; and \emph{$L$-Lipschitz} if $|f(\bm{\theta}) - f(\bm{\theta}')| \leq L \|\bm{\theta} - \bm{\theta}'\|$. 
We impose these regularity conditions on the sample loss:
\begin{assumption}[Loss regularity]\label{ass:loss_regularity}
    For any data point $\xi \in \mathbb{R}^s$, the function $\ell(\cdot, \xi)$ is $\mu$-strongly convex, $\beta$-smooth, and $L$-Lipschitz with respect to $\bm{\theta}$. 

\end{assumption}
These assumptions are standard in the machine unlearning literature, particularly when analyzing finite-time convergence alongside privacy guarantees~\citep{chourasia2023forget_unlearning, huang2023tight, allouah2024utility, van2025forget}. Under Ass.~\ref{ass:loss_regularity}, let $\vtheta^*$ (resp.~$\vtheta^*_r$) denote the unique minimizer of $\Lcal(\cdot,\D)$ (resp.~$\Lcal(\cdot, \Dr)$), and let $\kappa_{\ell}\coloneqq\beta/\mu$ denote the condition number. We write $\Fcal$ for the class of functions satisfying Ass.~\ref{ass:loss_regularity}.
Finally, note that $\mu$-strong convexity and $L$-Lipschitz continuity jointly imply that the loss has a  bounded effective domain, with diameter at most $ 4 L/\mu$.

In some of our results, we compare \VRU against \emph{arbitrary} retraining/unlearning algorithms.
This level of generality comes at the cost of additional, though still broad, assumptions on the distributions $\Dr$ and $\Df$.
\begin{assumption}[Distributional assumptions]\label{ass:flex}
    The retain distribution $\mathcal{D}_r$ is such that, for any 
    $p \in [0,1]$, there exists a measurable $A \subset \mathbb{R}^s$ with 
    $\mathbb{P}_{\xi \sim \mathcal{D}_r}[\xi \in A] = p$. 
    Moreover, there exist disjoint measurable sets $S_r,S_f\subseteq\mathbb R^s$ with
$\mathbb P_{\xi\sim\mathcal D_r}[\xi\in S_r]=1$ and $\mathbb P_{\xi\sim\mathcal D_f}[\xi\in S_f]=1$,
and $(S_r\cup S_f)^c\neq\emptyset$.
\end{assumption}
Note that Assumption~\ref{ass:flex} holds, for instance, when $\Dr$ is absolutely continuous and the supports of $\Dr$ and $\Df$ are disjoint and do not cover all of $\mathbb{R}^s$. This support separation arises naturally when retain and forget data are drawn from distinct subpopulations, as in class unlearning or toxic content removal.

\paragraph{Class of Unlearning Algorithms.} We formally define the class of unlearning algorithms considered in this work. 

\begin{definition}[Unlearning Algorithm]\label{def:unlearning_algo}
    An unlearning algorithm $\mathcal{U}: \mathbb{N} \times \Fcal \times \Pcal(\mathbb{R}^s) \times \Pcal(\mathbb{R}^s) \to \mathbb{R}^d$ is a (possibly randomized) procedure that takes as input (i) a number of iterations $T$, (ii) a loss function $\ell$, and (iii) access to samples from the retain distribution~$\mathcal{D}_r$, and  (iv) to samples from the forget distribution~$\mathcal{D}_f$. The output of the unlearning algorithm is a model $\bm{\theta}_T = \mathcal{U}(T, \ell, \mathcal{D}_r, \mathcal{D}_f) \in \mathbb{R}^d$. The algorithm is initialized at the optimum $\vtheta^*$ which we omit from the notation.
Initializing at the minimizer of $\Lcal(\cdot,\D)$ is standard in the $(\varepsilon,\delta)-$unlearning literature~\cite{guo2020certified, yi2024scalable, chien2024langevin, van2025forget}. We denote by $\mathbb{U}$ the class of all such unlearning algorithms.
\end{definition}
In contrast, a retraining algorithm $\mathcal{A}: \sN \times (\mathbb{R}^s)^* \to \mathbb{R}^d$ is a (possibly randomized) training procedure that takes as input the retraining time $T$, the retain distribution $\mathcal{D}_r$, and outputs a model $\bm{\theta}_T= \mathcal{A}(T, \mathcal{D}_r) \in \mathbb{R}^d$ 
trained exclusively on the retain data. The initialization of a retraining algorithm is random and will not be made explicit. Similarly, the number of iterations $T$ will be omitted when not of interest, for notation simplicity. We denote by $\mathbb{A}$ the class of all such unlearning algorithms.

We adopt the standard definition of ($\varepsilon, \delta)$-unlearning, first proposed by~\citet{ginart2019making}.
\begin{definition}[$(\varepsilon, \delta)$-unlearning]\label{def:unlearning}
    An unlearning algorithm $\mathcal{U} \in \mathbb{U}$ satisfies $(\varepsilon, \delta)$-unlearning if there exists a retraining algorithm $\mathcal{A}$ such that, for any distributions $(\mathcal{D}_r, \mathcal{D}_f)$, and any measurable subset $S \subseteq \mathbb{R}^d$,
    \begin{align*}
        \mathbb{P}[\Ucal(T, \ell, \Dr, \Df) \in S] &\leq e^{\varepsilon} \cdot \mathbb{P}[\mathcal{A}(\mathcal{D}_r) \in S] + \delta  , \\
        \mathbb{P}[\mathcal{A}(\mathcal{D}_r) \in S] &\leq e^{\varepsilon} \cdot \mathbb{P}[\Ucal(T, \ell, \Dr, \Df) \in S] + \delta .
    \end{align*}
\end{definition}
    In other words, an unlearning method achieves $(\varepsilon, \delta)-$unlearning if the distribution of models that unlearned $\Df$ is close to the distribution of models retrained from scratch on $\Dr$; the latter are perfectly private with respect to $\Df$ since they have never been exposed to it.

    We denote by $\mathbb{U}_{\varepsilon,\delta} \subseteq \mathbb{U}$ the class of all unlearning algorithms satisfying $(\varepsilon, \delta)$-unlearning, and by $\mathbb{U}^r_{\varepsilon,\delta} \subseteq \mathbb{U}_{\varepsilon,\delta}$ the subclass that does not access the forget set during the unlearning procedure. We refer to the latter as forget-set-free methods. 
    We further define the privacy budget $\kdp \coloneqq \varepsilon^{-1}\sqrt{2 \;\text{log}(2.5/\delta)}$ \cite{DP_book}, that links privacy levels to added noise in the algorithm.

\paragraph{Convergence times.} To evaluate the utility of unlearning algorithms, we adopt the standard notion of convergence time. This notion allows to link the largest accepted error of an unlearning procedure with the time it takes to achieve it \cite{allouah2024utility, van2025forget, zou2025certified}.
For a given error threshold $e > 0$ and an unlearning algorithm $\mathcal{U}\in\sU$, we define
\[
    T_e(\ell, \Ucal) \coloneqq \min_{T\in \mathbb{N}} \{T ; \; \mathbb{E}[\Lcal_r(\Ucal(T, \ell, \Dr, \Df)) - \Lcal_{r}(\vtheta^*_r)] \le e \},
\]
where $\Lcal_r(\cdot) \coloneqq \Lcal(\cdot, \Dr)$. Similarly, for a given error threshold $e > 0$ and a retraining algorithm $\mathcal{A}\in\sA$, we define
\[
    T_e(\ell, \Acal) \coloneqq \min_{T\in \mathbb{N}} \{T ; \; \mathbb{E}[\Lcal_r(\Acal(T, \Dr) - \Lcal_{r}(\vtheta^*_r)] \le e \}. 
\]

For our subsequent analysis, we define the error parameter 
\begin{equation}\label{eq:nu_T}
    \nu_T\coloneqq \sqrt{2 h(T,\delta)} L \frac{1+\kl}{\mu\sqrt{T}}\:,
\end{equation}
where $h(T,\delta) \coloneqq 1+ 624 \left(\log(\log(T)) +\log(2/\delta)\right)$. This parameter is used to bound, with high probability, the distance between a \VRU iterate and the optimum $\vtheta_r^*$.

\paragraph{Notation.} We denote by $B(\vtheta, R)$ the closed Euclidean ball of radius $R$ around $\vtheta$, and by $\operatorname{proj}_C(\vx) = \operatorname*{argmin}_{\vy \in C} \norm{\vx - \vy}_2$ the Euclidean projection of $\vx$ onto a closed convex set $C$.

\section{Main Results}
In this section, we first introduce our Variance-Reduced Unlearning (\VRU) algorithm and explain the intuition behind its gradient estimator (Section~\ref{subsec:VRU}). We then establish its convergence rate (Section~\ref{subsec:VRU_speed}), compare it to the state-of-the-art, proving significant speedup when compared with unlearning and retraining methods (Section~\ref{subsec:vs-speed}).  We finally provide a formal separation: \VRU provably outperforms \emph{any} $(\varepsilon, \delta)$-unlearning method that does not access the forget set (Section~\ref{subsec:vs-bound}).
All proofs are deferred to Appendix \ref{app:theo}.

\subsection{Introducing the \VRU Algorithm}\label{subsec:VRU}
We define the variance-reduced unlearning (\VRU) algorithm, and its associated stochastic gradient estimator,
\begin{equation}\label{eq:VRU_loss}
    \nablaT(\vtheta, \xi_r, \xi_f) = \nabla \ell (\vtheta,\xi^r) \underbrace{- \nabla \ell (\vtheta^*,\xi^r) - \tfrac{r_f}{1-r_f} \nabla \ell (\vtheta^*,\xi^f)}_{\text{correction term}} \:,
\end{equation}
where $\xi^f \sim \Df$ and $\xi^r \sim \Dr$ are i.i.d. samples drawn from the forget and retain distributions. 
The intuition behind this gradient estimator is relatively straightforward: to the usual stochastic gradient, we add a correction term with null expectation that significantly reduces the variance.

\paragraph{Structure of the \VRU algorithm.}
Our proposed algorithm operates in two phases. First, we apply the Projected Stochastic Gradient Descent (PSGD, see \eg Alg. 12.4 in \citet{handbook}) algorithm to our gradient estimator $\nablaT$ with decreasing step size $1/(\mu t)$. The projection is done onto the ball $B(\vtheta^*, \frac{r_f}{1-r_f}\frac{L}{\mu})$, which is guaranteed to contain the global optimum $\vtheta^*_r$ (see Lemma~\ref{lemma:close_opt}). Thus, the projection never moves the iterates away from the optimum. Finally, we apply noise to the optimization output
to ensure $(\varepsilon,\delta)-$unlearning.
Rapid convergence to the global minimizer $\vtheta_r^\star$ implies that only a small amount of noise is required: it scales with $r_f$ and decreases as $1/\sqrt{T}$ with the number of optimization steps $T$, as explained in Appendix~\ref{app:theo}.

\paragraph{Unbiasedness of $\,\nablaT(\vtheta).$} As $\vtheta^*$ is the minimizer of the loss on $\D$, at this specific point, the gradients on the retain and forget sets cancel each other out,
\begin{align}
    (1-r_f)&\E{\nabla \ell (\vtheta^*,\xi^r)} = (1-r_f)\nabla \Lcal(\vtheta^*, \Dr)  \\
    &= - r_f \nabla \Lcal(\vtheta^*, \Df)
    = - r_f \E{ \nabla \ell (\vtheta^*,\xi^f)}.\label{eq:unbiased}
\end{align}
This ensures that our correction term has zero expectation, keeping the gradient estimator unbiased.

\paragraph{Reduced variance of $\,\nablaT(\vtheta).$}
A natural approach to approximate $\vtheta^*_r$ starting from $\vtheta^*$ is to apply stochastic gradient descent on $\Dr$ using $\nabla\ell(\vtheta,\xi_r)$. While the expected norm of this gradient is small near $\vtheta^*$, at most $\frac{r_f}{1-r_f} L$ (\Eqref{eq:unbiased}), it may still exhibit significant variance even when $r_f$ is small. This leads to slow convergence when fine-tuning on $\Dr$ alone, a common approach in certified unlearning methods.

The challenge lies in reducing this variance without
having to compute full-batch gradients. 
To this end, we draw on the main idea behind the \textit{SVRG} algorithm \cite{johnson2013accelerating}, which replaces the stochastic gradient with the difference $\nabla\ell(\vtheta,\xi_r) - \nabla\ell(\vtheta^*,\xi_r)$. Since the loss is smooth, this difference has a low variance when $\vtheta$ and $\vtheta^*$ are close. However, this introduces a bias, which \textit{SVRG} corrects by adding a full-batch gradient $\nabla\mathcal{L}_r(\vtheta)$, a computationally expensive operation that must be periodically recomputed as the iterates move away from the anchor.

Our key observation is that we can instead add the stochastic term $-\frac{r_f}{1-r_f}\nabla\ell(\vtheta^*,\xi_f)$, which has the same expectation (\Eqref{eq:unbiased}). As $r_f$ represents a small portion of data, this term has much lower variance than $\nabla\ell(\vtheta^*,\xi_r)$, allowing it to correct the bias without compromising the variance reduction. This yields the \VRU update. Additionally, since $\vtheta^*$ and $\vtheta^*_r$ are close in parameter space (\lemref{lemma:close_opt}), the low-variance gradient information anchored at $\vtheta^*$ stays informative throughout the optimization trajectory. Unlike \textit{SVRG}, no periodic recomputation is needed. This lead to the faster convergence of \VRU iterates, and thus faster unlearning, as quantified in the next subsection.

\begin{algorithm}[t]
\caption[\textbf{VRU}]{\textbf{VRU} (Variance Reduced Unlearning)}
\label{alg:VRU}
\begin{algorithmic}[1]
\REQUIRE Number of iterations $T$, trained model $\vtheta^*$, loss function $\ell$, forget ratio $r_f \in (0,1)$, retain set $\Dr$, forget set $\Df$, privacy budget $\kdp$

\STATE Set $\theta_0 = \theta^*$, compute $\nu_T$  (\Eqref{eq:nu_T}) and $R\coloneqq\frac{r_f}{1-r_f}\frac{L}{\mu}$
\FOR{$t = 0$ to $T-1$}
    \STATE Sample data points: $\xi_t^{r}\sim \Dr$, $\xi_t^{f} \sim \Df$
    \STATE Compute variance-reduced gradient:
    \[
    \tilde{\nabla}_t \gets \nabla \ell(\vtheta_t, \xi_t^{r})
       - \nabla \ell(\vtheta^*, \xi_t^{r})
       - \frac{r_f}{1-r_f} \, \nabla \ell(\vtheta^*, \xi_t^{f}) 
    \]
    \STATE Update parameters and perform projection on the ball: $\vtheta_{t+1} \gets \operatorname{proj}_{\vtheta^*,R)}\left(\vtheta_t - 1/\mu t \, \tilde{\nabla}_t\right)$ \alglinelabel{eq:proj}
\ENDFOR
\STATE Sample $Z\sim\Ncal(0,1)$ 

\STATE  Noise the model to ensure unlearning: \[\tilde{\vtheta}_T = \vtheta_T + \left(\frac{r_f}{1-r_f}\right) \nu_T \kdp Z\] \alglinelabel{eq:proj}

\STATE \textbf{return} Final model $\tilde{\vtheta}_T$
\end{algorithmic}
\end{algorithm}

\subsection{Convergence Speed of the \VRU Algorithm}\label{subsec:VRU_speed}

The following theorem characterizes the complexity of \VRU in terms of the forget-set fraction $r_f$, the error threshold $e$, the privacy budget $\kdp$, and the local loss geometry $\kl$.
\begin{restatable}{theorem}{mainthm}\label{theo:VRU_speed}
     Let $\Fcal$ be the set of $\mu$-strongly-convex, $L$-Lipschitz, and $\beta$-smooth loss functions. Then, for any $\ell\in\Fcal$ and any $e>0$,
    \begin{equation}
            T_e(\ell, \VRU) = \tilde{\Ocal}\left(\kl^3(1+d\kdp^2\log(\frac{1}{\delta})) \frac{e_0}{e}\left(\frac{r_f}{1-r_f}\right)^2 \right)\:.
    \end{equation}

\end{restatable}
 \paragraph{Proof sketch.} We start by upper-bounding the Lipschitz constant of $\nablaT$ on a small ball around $\vtheta^*$, then leverage this regularity to measure the speed of almost-sure  convergence of the PSGD's final iterate to the global minimum $\vtheta^*_r$, before applying the Gaussian mechanism \cite{DP_book} to ensure unlearning. See App. \ref{app:theo} for the complete proof.

One highlight of this result is its  quadratic dependence on the forget fraction, $\Ocal(r_f^2)$, coupled with an $\Ocal(1/e)$ dependency on the excess risk.
Since $r_f$ is typically  small, as low as $\Ocal(1/n)$ for point-wise unlearning, this scaling translates into a substantial reduction in computational cost compared to full retraining or existing unlearning algorithms (see the next subsection).
The result is stated in $\tilde{\Ocal}$ rather than in $\Ocal$ due to an additional  factor in $\log \log T$ that arises when controlling the final iterate's distance to the global optimum with high probability.
Regarding the dependence on the condition number $\kappa_{\ell}$, we emphasize that this term reflects the \emph{local} geometry of the loss landscape. 
Indeed, the projection step (Line~\ref{eq:proj} in \VRU) confines the iterates to a ball of radius  $\frac{r_f}{1-r_f}\frac{L}{\mu}$ around  $\vtheta^*$, so the rate
is governed by the local curvature rather than the global condition number. Furthermore, the analysis is worst-case: in practice 
the global Lipschitz constant $L$ can effectively be replaced by the magnitude of the average forget set gradient $\norm{\nabla\Lcal(\vtheta^*,\Df)}$, yielding a tighter, bound that depends 
on the specific samples to unlearn. We discuss implementation details for efficiently exploiting these properties in Section \ref{subsec:VRU-practice}. 

\subsection{Improvement Over Existing Methods}\label{subsec:vs-speed}
The best previously known convergence rate for $(\varepsilon, \delta)$-unlearning of functions in $\Fcal$ is achieved 
by the ``Noise and Fine-Tune'' (\emph{NFT}) algorithm \citep{DescentToDelete}, which also exhibits 
a quadratic dependence on the forget fraction, $\mathcal{O}(r_f^2)$
\citep{van2025forget}, but has a worse dependency in $e$ of $\Ocal(1/e^2)$. This prevents \emph{NFT} from outperforming retraining when a low excess risk is required.
Theorem~\ref{theo:VRU_speed} 
establishes that \VRU improves this dependency to $\mathcal{O}(1/e)$, meaning that the speedup \VRU offers when compared to retraining does not decrease when dealing with smaller values of $e$. As a consequence, \VRU has an improved bound over \emph{NFT} for small values of $e$:
\begin{corollary}\label{corr:vru_vs_nft}
    Let $e>0$. Let $T^{\text{max}}_e(\VRU)$ (resp. $T^{\text{max}}_e(\mathit{NFT})$) be the best known asymptotical upper-bound on $T_e(\VRU)$ (resp. $T_e(\mathit{NFT})$). Then,
    \begin{equation}
        \frac{T^{\text{max}}_e(\VRU)}{T^{\text{max}}_e(\mathit{NFT})} = \Theta\left(\kl^3\log(1/\delta)\frac{e}{e_0}\right) \,.
    \end{equation}
\end{corollary}

Beyond improving upon existing unlearning methods, \VRU also compares favorably to retraining when the forget fraction is small:

\begin{corollary}
    \label{cor:beat_retrain}
    Under Assumption~\ref{ass:flex}, let $0<e<e_0$ and $\Acal\in\sA$. Then,
    \begin{equation}
        \frac{T_e(\VRU)}{T_e(\Acal)} = \tilde{\Ocal}\left(\kl^3\left(1+ d\kdp^2\log(\frac{1}{\delta})\right) \left(\frac{r_f}{1-r_f}\right)^2 \right)\,.
    \end{equation}
\end{corollary}
Corollary~\ref{cor:beat_retrain} addresses a natural question in machine unlearning: under what conditions can unlearning methods outperform retraining? The challenge lies in achieving favorable dependence on both the target excess risk $e$ and the forget fraction $r_f$ simultaneously. Retraining algorithms exhibit an $\mathcal{O}(1/e)$ dependence on excess risk, whereas \emph{NFT}, despite its advantageous $\mathcal{O}(r_f^2)$ scaling, suffers from $\mathcal{O}(1/e^2)$. This means \emph{NFT} loses its advantage over retraining in the low-error regime. By leveraging forget set gradients, \VRU combines the $\mathcal{O}(r_f^2)$ scaling of unlearning methods with the $\mathcal{O}(1/e)$ dependence of retraining, substantially extending the regime of $(e, r_f)$ pairs for which unlearning offers meaningful computational gains.

\subsection{Improvement Over Any Forget-Set-Free Method}\label{subsec:vs-bound}
In the following, we present a result which we argue is a fundamental performance characterization of \VRU over \emph{any} first-order forget-set-free $(\varepsilon, \delta)$-unlearning method.

A known limitation of first-order algorithms \emph{without access to the forget set} is their inability to outperform retraining from scratch for a certain range of excess risks $e$ \cite{van2025forget}. Crucially, \VRU sidesteps this barrier by incorporating forget set gradients at each iteration. Theorem~\ref{theo:VRU-beats} formalizes the resulting separation.

When optimizing a function in a certain set (\ie $\Fcal$), one wants algorithms guaranteed to perform well regardless of the function chosen from the set. This motivates studying not how a specific algorithm performs when ran on a specific loss, but rather how fast can a specific algorithm optimize any loss in the class. We thus define the worst-case convergence time as a measure of the performance of an algorithm 
\[
T_e(\Ucal) \coloneqq \sup_{\:\:\ell\in\Fcal}T_e \: (\ell,\Ucal) \:.
\]
With this quantity introduced, we are ready to showcase the performance gap between first-order $(\varepsilon, \delta)$-unlearning methods, depending on whether they leverage the forget set.

\begin{theorem}[Fundamental gain from forget set access]\label{theo:VRU-beats}
Assume Assumptions \ref{ass:loss_regularity} and \ref{ass:flex} hold.

For any $\delta_{min}>0$, there is a constant $c>0$ such that, for any forget-set-free unlearning algorithm $\Ucal\in\sU_{\varepsilon,\delta}^r$, if $\delta \in[\delta_{min},\varepsilon]$ and $e< c \kdp^2 \left(r_f/(1-r_f)\right)^2 e_0$,
\[
\frac{T_e(\VRU)}{T_e(\Ucal)} = \tilde{\Ocal}\left(d\kdp^2 \log\left(1/\delta\right)\kl^3\left(\frac{r_f}{1-r_f}\right)^2\right)\,.
\]

\end{theorem}
\begin{corollary}
    Under the assumption of \theoref{theo:VRU-beats}, for any forget-set-free unlearning algorithm $\Ucal\in\sU_{\varepsilon,\delta}^r$,
    \[
    \liminf_{r_f\to 0} \frac{T_e(\VRU)}{T_e(\Ucal)} = 0 \,.
    \]
\end{corollary}
This theorem has important implications, as it characterizes the complexity gain achieved by relying on the forget set, while comparing the speed between first-order $(\varepsilon, \delta)$-unlearning algorithms. More precisely, it compares the speed of one specific algorithm that uses the forget set, \VRU, to the speed of any that does not. Thus, in the error regime described in the theorem, any first-order algorithm that does not rely on the forget set will perform worse than \VRU on the hardest losses in $\Fcal$. This represents an important advantage in the typical use-cases of unlearning, characterized by small values of $r_f$ and $e$.
\subsection{Discussion}
Recent work by \citet{mavrothalassitis2025ascent} suggests that standard descent-ascent strategies (combining gradient ascent on $\Df$ with gradient descent on $\Dr$) may degrade performance relative to the original model $\vtheta^*$ and fail to converge to the optimum $\vtheta^*_r$. However, the \VRU algorithm lies outside the scope of this negative result due to structural differences: unlike the analyzed methods, \VRU incorporates gradients computed at $\vtheta^*$ and utilizes a three-terms update rule. More precisely, \citet{mavrothalassitis2025ascent} demonstrate that for specific ranges of values of $r_f$ and $e$, descent-ascent cannot outperform retraining from scratch in regularized logistic regression settings. This stands in direct contrast to our Corollary \ref{cor:beat_retrain}, which establishes values of $r_f$ for which \VRU outperforms retraining, regardless of $e$.

\section{Experiments}\label{sec:exp}
We empirically evaluate \VRU against certified unlearning algorithms, empirical methods, and retraining baselines. Section~\ref{subsec:methods} first describes all compared methods, and Section~\ref{subsec:VRU-practice} details how to efficiently implement \VRU in practice. Section~\ref{subsec:certified} then compares our approach to certified methods and retraining baselines (\Figref{fig:cvx}). Finally, Section~\ref{subsec:empirical} benchmarks against empirical unlearning algorithms, evaluating privacy leakage via membership inference attacks, as well as utility (\Figref{fig:empirical}). To compare \VRU against other methods and assess the validity of our theory without introducing unnecessary complexity, we consider a logistic regression task that ensures strong convexity.

\subsection{Compared Methods}\label{subsec:methods}
To evaluate the performance of \VRU relative to existing approaches, we compare with the $(\varepsilon,\delta)$ literature, the empirical literature, and the retraining one. We choose the ``Noise and Fine-Tune" (\textbf{NFT}) \cite{DescentToDelete} algorithm as representative of the $(\varepsilon,\delta)-$unlearning methods. \textbf{NFT} achieves the best known utility-privacy tradeoff on losses in $\Fcal$, as quantified in recent studies \cite{allouah2024utility, van2025forget}. For empirical methods, we evaluate against: the \textbf{SCRUB} algorithm \cite{kurmanji2024towards_unbounded}, which alternates between maximizing the KL divergence on the forget set relative to the original model (the teacher) and minimizing the divergence from the teacher on the retain set, combined with a data fidelity term; the \textbf{NegGrad+} baseline \cite{kurmanji2024towards_unbounded} that alternates between gradient ascent steps on $\Df$ and descent steps on $\Dr$, always concluding with descent on $\Dr$ to preserve utility;  the \textbf{Fine-Tune} baseline that only performs gradient descent on the retain set. Finally, we use the \textbf{GD}, \textbf{SGD} and \textbf{SVRG} \cite{johnson2013accelerating} methods as retraining baselines.

In the following experiments, the various methods are always compared with an equal computational budget, measured in the number of sample gradients. For instance, as \VRU requires several gradient passes at each batch, it has a larger per-epoch computational cost than \textbf{NFT} and is thus ran for less epochs.

\subsection{Implementing the \VRU Method in Practice.}\label{subsec:VRU-practice}

When implementing \VRU, several elements can be adapted to improve efficiency. While some conservative choices in Algorithm~\ref{alg:VRU} ensure worst-case convergence guarantees, loss-specific adjustments can be made at run time without compromising on $(\varepsilon,\delta)$-unlearning. First, since the forget set is typically small, replacing the stochastic forget gradient in \Eqref{eq:VRU_loss} with a \emph{full-batch} gradient $\nabla\Lcal(\vtheta^*, \Df)$ computed once before optimization begins reduces computational cost whenever the unlearning computational budget exceeds roughly $1/r_f<1$ fine-tuning epochs on the retain set. This substitution preserves unbiasedness and further reduces variance, as the full-batch gradient is deterministic. Additionally, we can use the gradient's norm $\norm{\nabla\Lcal(\vtheta^*, \Df)}_2$ to replace the Lipschitz constant $L$ in the algorithm, updating it to Alg \ref{alg:VRU-exp},
as proven is App. \ref{app:VRU-exp}. This is advantageous because $L$ is typically large and NP-hard to compute for neural networks~\citep{virmaux2018lipschitz}. 
By showing that the gradient norm suffices, we allow for both wider applicability and faster convergence. Further details about how to implement \VRU in practice are available in App. \ref{app:VRU-exp}.

Beyond these run-time adjustments that avoid reliance on $L$ and improve convergence speed, \VRU has several built-in advantages over its empirical peers, particularly those relying on forget set gradient ascent.

First, \VRU does not require any hyperparameter to control the strength of gradient ascent. Empirical methods such as \textbf{SCRUB} or \textbf{NegGrad+} use biased gradient estimators: the optimum $\vtheta^*_r$ does not represent a stationary point for them, as the average gradient on $\Df$ is generally non-null at this point. Consequently, they must weigh their forget set gradient ascent with a carefully tuned rescaling factor in order to approximate the optimum a time-consuming process. \VRU sidesteps this issue entirely, as its gradient expectation is null at $\vtheta^*_r$, and it has no tunable ascent hyperparameter. This stability guarantee enables convergence to $\vtheta^*_r$ with arbitrarily high probability, unlike other ascent-based methods, which are known for their instability and potential for divergence. In contrast, running \VRU longer always decreases the expected distance to the optimum.

Second, the only choice one must make when implementing \VRU is the privacy budget $\kdp$. A larger value of $\kdp$ yields stronger privacy guarantees but reduced utility. Importantly, this choice can be made \emph{a posteriori}: because the noising step is the final component of the algorithm, practitioners can decouple the optimization from the privacy decision and easily simulate several levels of noise without retraining, to find their application-specific sweetspot.
\begin{figure}[t]
    \centering
    \includegraphics[width=0.5\textwidth]{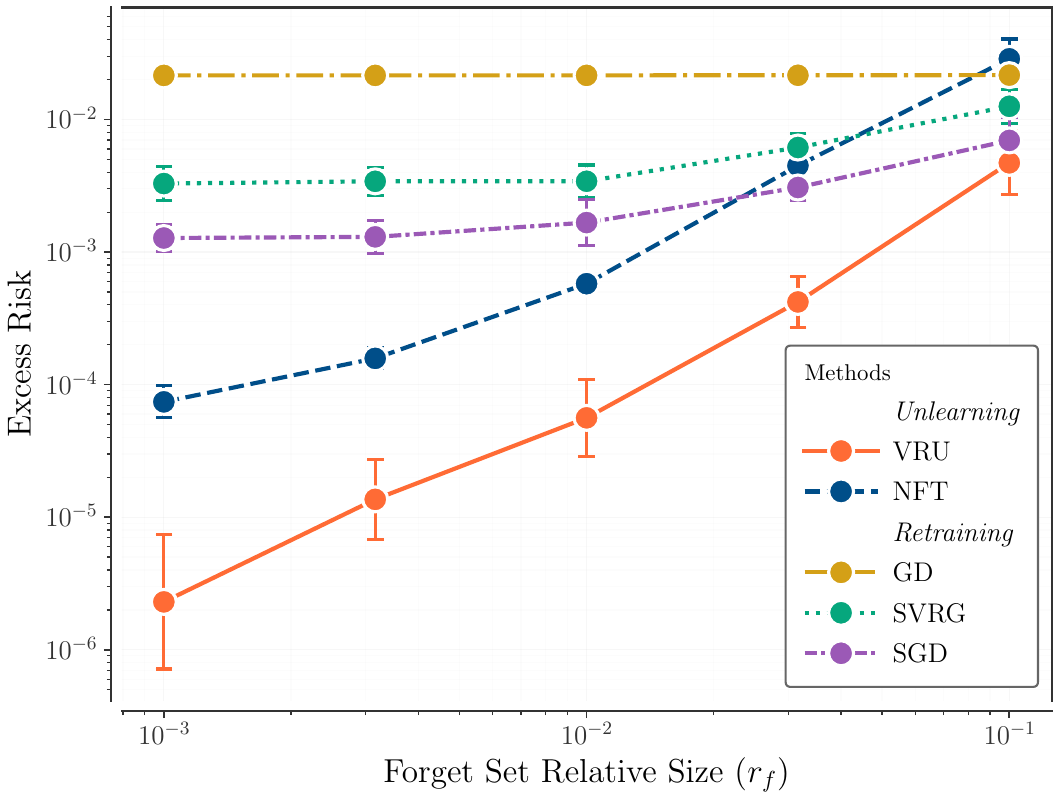}
    \caption{Excess risk of certified unlearning and retraining methods for 
    varying forget fractions $r_f$, under fixed computational budget (10 epochs) 
    and privacy budget ($\kdp = 1$). Results are averaged over 30 runs; error bars indicate $\pm 1$ standard deviation. \VRU achieves the 
    lowest excess risk across all tested $r_f$ values, with gains exceeding an 
    order of magnitude for $r_f < 10^{-2}$.}
    \label{fig:cvx}
\end{figure}
\subsection{Comparing to Certified Approaches}\label{subsec:certified}
We first evaluate $(\varepsilon, \delta)$-unlearning methods alongside retraining 
baselines, which offer perfect privacy by training exclusively on 
$\mathcal{D}_r$. We consider a logistic regression task with cross-entropy loss 
and $\ell_2$ regularization to ensure strong convexity, using the Digit 
dataset~\citep{Digit}. Full experimental details are provided in Appendix~\ref{app:exp_details}.

\paragraph{Setup.} We measure the excess risk achieved by each method across 
$5$ values of $r_f$ spread logarithmically between \(10^{-3}\) and $10^{-1}$, under 
a computational budget equivalent to $10$ epochs of retraining from scratch with SGD and a 
privacy budget of $\kdp = 1$. For each trial, the forget set is selected 
uniformly at random and re-sampled across seeds. Results are averaged over 30 
seeds; error bars in Figure~\ref{fig:cvx} indicate $\pm 1$ 
standard deviation. Since computing the bounded sensitivity~\citep{DP_book} is 
intractable in most experimental settings, we measure it directly and provide it to all 
methods requiring it, ensuring privacy guarantees. We discuss practical 
implementation details for \VRU in Section~\ref{subsec:VRU-practice}.

\paragraph{Results.} Figure~\ref{fig:cvx} shows that for $r_f \leq 0.1$, 
\VRU achieves a lower excess risk than all competing methods, with 
the performance gap widening as $r_f$ decreases, nearing two orders of 
magnitude at $r_f = 10^{-3}$. This behavior aligns with our theoretical 
predictions as we have a better dependence on $e$ than \textbf{NFT}, making smaller values of $e$ reachable under a given computational constraint, in the error regime described in Corr. \ref{corr:vru_vs_nft}.

\subsection{Comparing to Empirical Approaches}\label{subsec:empirical}
\begin{figure}[t]
    \centering
    \includegraphics[width=0.5\textwidth]{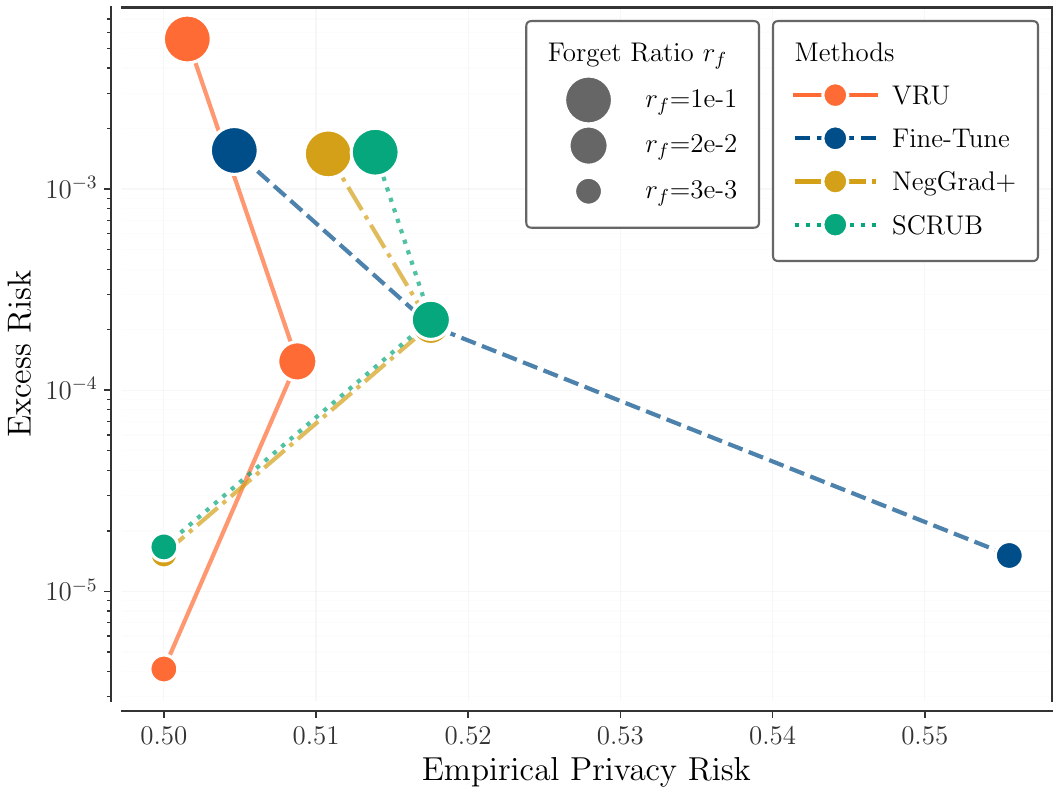}
    \caption{Privacy-utility trade-off under fixed computational budget (5 
    epochs). Each point represents one method at a given $r_f$ value. 
    \textbf{Excess risk} (y-axis): lower is better. \textbf{Empirical privacy 
    risk} (x-axis): MIA accuracy, 50\% indicates perfect unlearning. 
    The lower-left region represents the ideal trade-off.}
    \label{fig:empirical}
\end{figure}

Unlike certified methods, empirical unlearning algorithms lack formal privacy 
guarantees, necessitating empirical evaluation of privacy leakage. Following 
standard practice~\citep{carlini2022membership, hayes2024false_sense}, we 
measure privacy risk via membership inference attacks (MIAs), which assess 
whether an adversary can distinguish forgotten samples from unseen test 
samples, a successful distinction indicating incomplete unlearning. Specifically, 
we implement U-LiRA~\citep{hayes2024false_sense}, the unlearning-adapted variant 
of the LiRA attack~\citep{carlini2022membership}. We report MIA accuracy as our 
measure of \emph{empirical privacy risk}: an accuracy of 50\% corresponds to 
random guessing (perfect unlearning), while higher values indicate privacy 
leakage. Implementation details are provided in Appendix~\ref{app:exp_details}.

\paragraph{Setup.} We evaluate all methods under a computational budget of 5 
epochs across forget fractions $r_f \in \{3\times 10^{-3}, 2\times10^{-2},  
10^{-1}\}$. For each method, we report both the excess risk 
$\mathcal{L}_r(\vtheta) - \mathcal{L}_r(\vtheta^*_r)$ and the empirical privacy risk 
(MIA accuracy). Results are averaged over 3 independent runs.

\paragraph{Results.} Figure~\ref{fig:empirical} presents the privacy-utility 
trade-off for each method, where the lower-left corner represents the ideal 
outcome (low excess risk, low privacy leakage). 

For all tested values of $r_f$, 
\VRU achieves the lowest empirical privacy risk among all methods.
We observe that MIA accuracy remains close to 50\% for most methods across 
all settings. This reflects a known limitation of membership inference 
attacks: they were developed primarily for complex, overparameterized models 
where memorization is prevalent~\citep{carlini2022membership}, and their 
discriminative power diminishes in strongly convex settings where all methods 
converge toward the unique optimum. Nevertheless, considerable differences appear. 
The \textbf{Fine-Tune} baseline exhibits the highest privacy leakage, 
particularly at small $r_f$. Unlike other methods, \textbf{Fine-Tune} lacks 
of mechanisms to actively degrade performance on the forget set, leaving its 
loss on forgotten samples low and thus vulnerable to attack.

Regarding utility, \VRU achieves the lowest excess risk for small 
and moderate values of $r_f$. At $r_f = 0.1$, however, \VRU incurs 
higher excess risk than empirical methods as it is the only method to include a noise addition step, and the noise 
scales with $r_f$. This trade-off is 
expected: as $r_f$ increases, the computational advantage of unlearning over 
full retraining diminishes (cf.\ Figure~\ref{fig:cvx}), making $r_f \geq 0.1$ a 
less compelling regime for unlearning algorithms in general. In such cases, 
practitioners may prefer full retraining, which achieves both perfect privacy 
and comparable utility.
\section{Conclusion}

In this work, we introduced Variance-Reduced Unlearning (\VRU), the first first-order $(\varepsilon, \delta)$-unlearning algorithm to incorporate forget set gradients into its optimization process. By anchoring a variance-reduction mechanism at the pre-trained model, VRU bridges the gap between certified methods, which rely primarily on descent on the retain set, and empirical approaches, which also exploit the forget set but lack formal privacy guarantees.

Our convergence analysis for strongly convex, smooth, and Lipschitz losses establishes an $\mathcal{O}(1/e)$ dependence on the target excess risk, improving upon the $\mathcal{O}(1/e^2)$ scaling of prior certified methods and yielding a strictly larger efficiency regime compared to retraining from scratch. Beyond these improved rates, we proved a fundamental property: for a given range of error and forget ratio, \VRU asymptotically outperforms any first-order $(\varepsilon, \delta)$-unlearning algorithm that does not access the forget set. Experiments on strongly convex objectives corroborate our theoretical findings, demonstrating consistent gains over both certified and empirical baselines.

Our analysis relies on strong convexity and, in principle, on the knowledge of the exact pre-trained optimum $\vtheta^*$, though our experiments suggest robustness to inexact initialization. Additionally, while our bounds depend only on the local condition number, this dependence may still be restrictive in certain settings. A natural direction for future work is to extend variance-reduced unlearning to relaxed notions of convexity that better capture the behavior of neural networks near local optima, such as the Polyak-Łojasiewicz condition~\citep{karimi2016linear, liu2022loss} or the neural tangent kernel regime~\citep{jacot2018neural}.
\bibliography{main}
\newpage
\appendix
\onecolumn
\section{Proofs of main theoretical results.}\label{app:theo}

We start by recalling this lemma from the literature, allowing to bound the distance between the starting point of our unlearning procedure and the global optimum.

\begin{lemma}[Lemma C.2, \citet{van2025forget}]
\label{lemma:close_opt}
\begin{equation}
    \norm{\vtheta^*-\vtheta^*_r} \le \frac{r_f}{1-r_f}\cdot\frac{L}{\mu}\,.
\end{equation}
\end{lemma}

We then bound the Lipschitz constant of \VRU's stochastic gradient estimator around $\vtheta^*$, allowing for a better characterization of the local properties of the loss.
\begin{lemma}\label{lemma:small_grad}
For any $\vtheta \in B\left(\vtheta^*, \frac{r_f}{1-r_f}\frac{L}{\mu}\right)$ and $t\in\sN^*$, 
\begin{equation}
    \norm{\nablaT_t(\vtheta)} \le (1+\kl)\frac{r_f}{1-r_f}L,
\end{equation}
where $\kl$ is the condition number of the loss $\ell$.
\end{lemma}

\begin{proof}
\begin{align}
    \norm{\nablaT_t(\vtheta)} &\le \norm{\nabla\ell(\vtheta, \xi^t_r) - \nabla\ell(\vtheta^*, \xi^t_r)} + \left(\frac{r_f}{1-r_f}\right) \norm{\nabla\ell(\vtheta^*, \xi^t_f)} \\
    &\le \beta\norm{\vtheta-\vtheta^*} + \left(\frac{r_f}{1-r_f}\right)L\\
    &\underset{(1)}{\le} \left(\frac{r_f}{1-r_f}\right) L (1+\kl)\:,
\end{align}
where (1) is obtained through Lemma \ref{lemma:close_opt}.
\end{proof}

We are now in a position to prove the theorem.

\mainthm*
\begin{proof}[Proof, \theoref{theo:VRU_speed} ]
Applying the \VRU algorithm for $T$ iterations is equivalent to applying $T$ iterations of the PSGD algorithm (see \eg 12.4 in \citet{handbook}) to the stochastic variance-reduced gradient $\nablaT_t$ with step size $1/\mu t$ then applying Gaussian noise with magnitude $\left(\frac{r_f}{1-r_f}\right)\nu_T\kdp$ (see Eq. \ref{eq:nu_T}). We remind that the stochastic gradient estimator $\nablaT$ is unbiased (\Eqref{eq:unbiased}). Using Lemma \ref{lemma:small_grad}, we bound the gradient of any sample with probability $1$ and can thus apply
Proposition 1 in \citet{rakhlin2011making}. Hence, with probability at least $1-\delta/2$,
\begin{align}
    \norm{\vtheta_T-\vtheta^*_r}^2 &\le (624\log(2\log(T)/\delta)+1)\left(\frac{r_f}{1-r_f}\right)^2L^2\frac{(1+\kl)^2}{\mu^2 T} \label{eq:almost_sure}\\
    &\le \frac{1}{2}\left(\frac{r_f}{1-r_f}\right)^2 \nu_T^2 
\end{align}
where $\nu_T\coloneqq \sqrt{2h(T,\delta)} L \frac{1+\kl}{\mu\sqrt{T}}$ and $h(T,\delta) \coloneqq 624 \log(2\log(T)/\delta)+1\:.$

This sensitivity bound is somewhat unusual: it only holds with high probability. We show in \lemref{lemma:high_proba_DP} that applying an $(\varepsilon,\delta_1)-$DP mechanism when the sensitivity is bounded with probability $1-\delta_2$ rather than deterministically achieves $(\varepsilon,\delta_1+\delta_2)-$DP. We apply this result with $\delta_1=\delta_2=\delta/2$.
We thus define the noised estimate $\thetilde_T \coloneqq \vtheta_T + \left(\frac{r_f}{1-r_f}\right) \kdp \nu_T Z$, where $Z\sim\Ncal(0,1)$, achieving $(\varepsilon,\delta)-$unlearning through the standard Gaussian mechanism \cite{DP_book}.

We can now evaluate the loss of the noised model:
\begin{align}
\E{\Lcal_r(\Tilde{\vtheta}_T)-\Lcal_r(\vtheta^*_r)} &\le \E{\Lcal_r(\vtheta_T) - \Lcal_r(\vtheta^*_r)} + \frac{\beta}{2} \sE \norm{\Tilde{\vtheta}_T-\vtheta_T}^2 \\
&\underset{(1)} {\le} \frac{2}{T} \kl e_0\left(\frac{r_f}{1-r_f}\right)^2(1+\kl)^2 + 
\kl(1+\kl)^2 d \kdp^2 h(T,\delta) \left(\frac{r_f}{1-r_f}\right)^2 \frac{e_0}{2T} \\
&\le (1+\kl)^2 \frac{e_0}{2T}\left(\frac{r_f}{1-r_f}\right)^2\left(4\kl+d\kl\kdp^2 h(T,\delta)\right)\:.
\end{align}
(1): The first term is obtained through \citet{rakhlin2011making}'s Theorem 1.

The inequality can be rewritten as 
$T \le a + b \log \log T$, with $a>0$ and $b>0$. For $T\ge \exp(1)$, it holds
\begin{align}
T & \le a + b \log \log T  \le a + b\log T\\
  & \le (a+b) \log T.
\end{align}
From Lemma A.1 in \citet{shalev_book}, it follows that
\begin{equation}
    T \le 2(a+b) \log(a+b).
\end{equation}
We can then conclude:
\begin{equation}
    T_e(\ell, \VRU) = \tilde{\Ocal}\left( \kl^3(1+d\kdp^2 \log \left(1/\delta\right)) \frac{e_0}{e}\left(\frac{r_f}{1-r_f}\right)^2 \right)
\end{equation}

\end{proof}

Corollary \ref{corr:vru_vs_nft} is obtained directly by dividing the result in \theoref{theo:VRU_speed} by the one obtained in Theorem 3 in \citet{van2025forget}.

Let $\Ucal\in\sU^r_{\varepsilon,\delta}$ be an unlearning algorithm that does not access the forget set gradient.
Let $\Acal\in\sA$ be retraining from scratch on the retain set $\Dr$ with the PSGD algorithm, as defined in Alg. 12.4 of \citet{handbook}.

We recall the next two results, describing the speed of $\Acal$ and comparing it to methods in $\sU^r_{\varepsilon,\delta}$ and in $\sA$.

\begin{theorem}[Theorem 2, \citet{van2025forget}]\label{theo:lbunlearn}

Let $\delta \in[10^{-8},\varepsilon]$. Under Assumption \ref{ass:flex}, for any $\delta_{min}>0$, there exists a universal constant $c > 0$ such that, if $e < \min\left\{1,\,c \left(\frac{r_f}{1-r_f}\right)^2 \left(1+\kdp^2\right)\right\} e_0$, then, for any $\Ucal\in\sU^r_{\varepsilon,\delta}$,
\begin{equation}
\frac{T_e(\Ucal)}{T_e(\Acal)} = \Omega(1)\,.
\end{equation}
\end{theorem}

\begin{lemma}[Lemma 4.2, \citet{van2025forget}]
\label{lemma:scratch_speed} 
Under Assumption \ref{ass:flex}, and if $e<e_0$, we have
\begin{equation}
T_e(\Acal) = \Theta\left( \frac{e_0}{e} \right)\,
\end{equation}
\end{lemma}

\begin{proof}[Proof of Cor. \ref{cor:beat_retrain}:]
    We divide the upper bound in 
    \theoref{theo:VRU_speed} by the lower bound in

    \lemref{lemma:scratch_speed} and 
    the result follows.
\end{proof}

We now have all the building blocks necessary to prove \theoref{theo:VRU-beats}.

\begin{proof}[Proof of \theoref{theo:VRU-beats}]
    Let $c$ be the constant in \theoref{theo:lbunlearn}. 

    Let
    $e< c \kdp^2 \left(r_f/(1-r_f)\right)^2 e_0$ and $\delta\in[0,\eps]$. Let $\Ucal\in\sU^r_{\varepsilon,\delta}$. We control the speed of \VRU compared to $\Acal$ through Cor.~\ref{cor:beat_retrain}. We control the speed of $\Ucal$ compared to $\Acal$ through \theoref{theo:lbunlearn}.

    Then, 
    \begin{equation}
        \frac{T_e(\VRU)}{T_e(\Ucal)} = \frac{T_e(\VRU)}{T_e(\Acal)} \frac{T_e(\Acal)}{T_e(\Ucal)} = \tilde{\Ocal}\left(\left(1+d\kdp^2 \log\left(\frac{1}{\delta}\right)\right)\kl^3\left(\frac{r_f}{1-r_f}\right)^2\right)\:.
    \end{equation}
\end{proof}

\paragraph{Technical lemma.}
We show how a high-probability bound on the sensitivity can still translate to $(\varepsilon,\delta)-$differential privacy by adding the failure probabilities of the bound and the DP. While we are probably not the first to prove this result, we were unable to find a direct formulation elsewhere.

\begin{lemma}[Differential Privacy under High-Probability Sensitivity]
\label{lemma:high_proba_DP}
Suppose that a sensitivity bound $\Delta$ holds with probability at least $1 - \delta_1$. If a mechanism $\mathcal{M}$ satisfies $(\varepsilon, \delta_2)$-DP when the sensitivity is at most $\Delta$, then $\mathcal{M}$ satisfies $(\varepsilon, \delta_1 + \delta_2)$-DP.
\end{lemma}

\begin{proof}
Let $E$ be the event that the sensitivity bound holds, so $\Pr[E] \ge 1 - \delta_1$. For any measurable $S$:
\begin{align*}
\Pr[\mathcal{M}(D) \in S] 
&= \Pr[\mathcal{M}(D) \in S \mid E]\Pr[E] + \Pr[\mathcal{M}(D) \in S \mid E^c]\Pr[E^c] \\
&\le \bigl(e^\varepsilon \Pr[\mathcal{M}(D') \in S \mid E] + \delta_2\bigr)\Pr[E] + \Pr[E^c] \\
&\le e^\varepsilon \Pr[\mathcal{M}(D') \in S] + \delta_1 + \delta_2. \qedhere
\end{align*}
\end{proof}

\section{Run-time improvements of \VRU}\label{app:VRU-exp}

\begin{algorithm}[t]
\caption[\textbf{VRU-exp}]{\textbf{VRU-exp} (Variance Reduced Unlearning, experiments version)}
\label{alg:VRU-exp}
\begin{algorithmic}[1]
\REQUIRE Number of iterations $T$, trained model $\vtheta^*$, loss function $\ell$, forget ratio $r_f \in (0,1)$, retain set $\Dr$, forget set $\Df$, privacy budget $\kdp$

\STATE Set $\theta_0 = \theta^*$, 
\STATE Compute a full-batch gradient on the forget set: $\nabla^*_f \coloneqq \Lcal(\vtheta^*,\Df)$
\STATE Compute the improved error value $\nu_T^{\text{exp}}$ (\Eqref{eq:nu_T_exp}).
\STATE Set $R^{\text{exp}} \coloneqq\frac{r_f}{1-r_f}\frac{\norm{\nabla^*_f}}{\mu}$
\FOR{$t = 0$ to $T-1$}
    \STATE Sample data point: $\xi_t^{r}\sim \Dr$,
    \STATE Compute variance-reduced gradient:
    \[
    \tilde{\nabla}_t \gets \nabla \ell(\vtheta_t, \xi_t^{r})
       - \nabla \ell(\vtheta^*, \xi_t^{r})
       - \frac{r_f}{1-r_f} \, \nabla^*_f  
    \]
    \STATE Update parameters and perform projection on the ball: $\vtheta_{t+1} \gets \operatorname{proj}_{B(\vtheta^*,R^{\text{exp}})}\left(\vtheta_t - 1/\mu t \, \tilde{\nabla}_t\right)$
\ENDFOR
\STATE Sample $Z\sim\Ncal(0,1)$ 

\STATE Noise the model to ensure unlearning: \[\tilde{\vtheta}_T = \vtheta_T + \left(\frac{r_f}{1-r_f}\right) \nu^{\text{exp}}_T \kdp Z\]

\STATE \textbf{return} Final model $\tilde{\vtheta}_T$
\end{algorithmic}
\end{algorithm}

We introduce the following results with practical implementation of  in mind. They providing guidance for implementing \VRU more efficiently while preserving its theoretical guarantees. In particular, Alg. \ref{alg:VRU-exp} describes a practical implementation of \VRU that avoids requiring the Lipschitz constant $L$, which is often intractable to compute in practical settings. 
The convergence speed of  Alg. \ref{alg:VRU-exp} to $\vtheta^*_r$ can be proven by replacing \lemref{lemma:close_opt}

by \lemref{lemma:close_opt_implem}, and \lemref{lemma:small_grad} by \lemref{lemma:small_grad_exp} in the proof of \theoref{theo:VRU_speed}.
The update rule for Alg. \ref{alg:VRU-exp} thus becomes 
\begin{equation}\label{eq:VRU_loss}
    \nablaT^{\text{exp}}(\vtheta, \xi^r) = \nabla \ell (\vtheta,\xi^r) - \nabla \ell (\vtheta^*,\xi^r) - \frac{r_f}{1-r_f} \nabla \Lcal (\vtheta^*,\Df)\:, 
\end{equation}
where the full-batch gradient on $\Df$ is only computed once, before optimization begins.

The following result is similar to \lemref{lemma:close_opt},

but leverages specific forget set gradient, not the worst-case Lipschitz bound,
\begin{lemma}[Bounded optima distance]
\label{lemma:close_opt_implem}
\begin{equation}
    \norm{\vtheta^*-\vtheta^*_r} \le \frac{r_f}{1-r_f}\cdot\frac{\norm{\nabla\Lcal(\vtheta^*, \Df)}}{\mu}\,.
\end{equation}
\end{lemma}
\begin{proof}
    By strong convexity of $\Lcal(\cdot, \Dr)$, we have $\mu\norm{\vtheta^*-\vtheta^*_r} \le \norm{\nabla\Lcal_r(\vtheta^*)}$. 
    Additionally, $\norm{\nabla\Lcal(\vtheta^*, \Dr)} =  -\frac{r_f}{1-r_f}\norm{\nabla\Lcal(\vtheta^*, \Df)}$. This concludes the proof.
\end{proof}

 Instead of bounding $\norm{\nablaT_t(\vtheta)}$ on $B\left(\vtheta^*, \frac{r_f}{1-r_f}\frac{L}{\mu}\right)$, we can reduce the radius of the ball to $B\left(\vtheta^*, \frac{r_f}{1-r_f}\frac{\norm{\Lcal(\vtheta^*,\Df)}}{\mu}\right)$ (see \lemref{lemma:close_opt_implem}). This also allows for a reduction in the gradient's bound, as described in the following result.

\begin{lemma}\label{lemma:small_grad_exp}
For any $\vtheta \in B\left(\vtheta^*, \frac{r_f}{1-r_f}\frac{\norm{\Lcal(\vtheta^*,\Df)}}{\mu}\right)$ and $t\in\sN^*$, 
\begin{equation}
    \norm{\nablaT_t(\vtheta)} \le \frac{r_f}{1-r_f}(1+\kl)\norm{\nabla\Lcal(\vtheta^*,\Df)}\:,
\end{equation}
where $\kl$ is the condition number of the loss $l$, and the supremum is defined as the set is non-empty and upper-bounded by $L$.
\end{lemma}
\begin{proof}
\begin{align}
    \norm{\nablaT_t(\vtheta)} &\le \norm{\nabla\ell(\vtheta, \xi^t_r) - \nabla\ell(\vtheta^*, \xi^t_r)} + \left(\frac{r_f}{1-r_f}\right) \norm{\nabla\Lcal(\vtheta^*, \Df)} \\
    &\le \beta\norm{\vtheta-\vtheta^*} + \left(\frac{r_f}{1-r_f}\right)\norm{\nabla\Lcal(\vtheta^*, \Df)}\\
    &\underset{(2)}{\le} \frac{r_f}{1-r_f}(1+\kl)\norm{\nabla\Lcal(\vtheta^*,\Df)}  \:,
\end{align}
where (2) is obtained through Lemma \ref{lemma:close_opt_implem}.
\end{proof}

To simplify the expression of the VRU empirical algorithm, we define
\begin{equation}\label{eq:nu_T_exp}
    \nu^{\text{exp}}_T\coloneqq \frac{\sqrt{2h(T,\delta)}}{\mu\sqrt{T}}\norm{\nabla\Lcal(\vtheta^*,\Df)}
    (1+\kl).
\end{equation}

\begin{figure}
    \centering
    \includegraphics[width=0.5\textwidth]{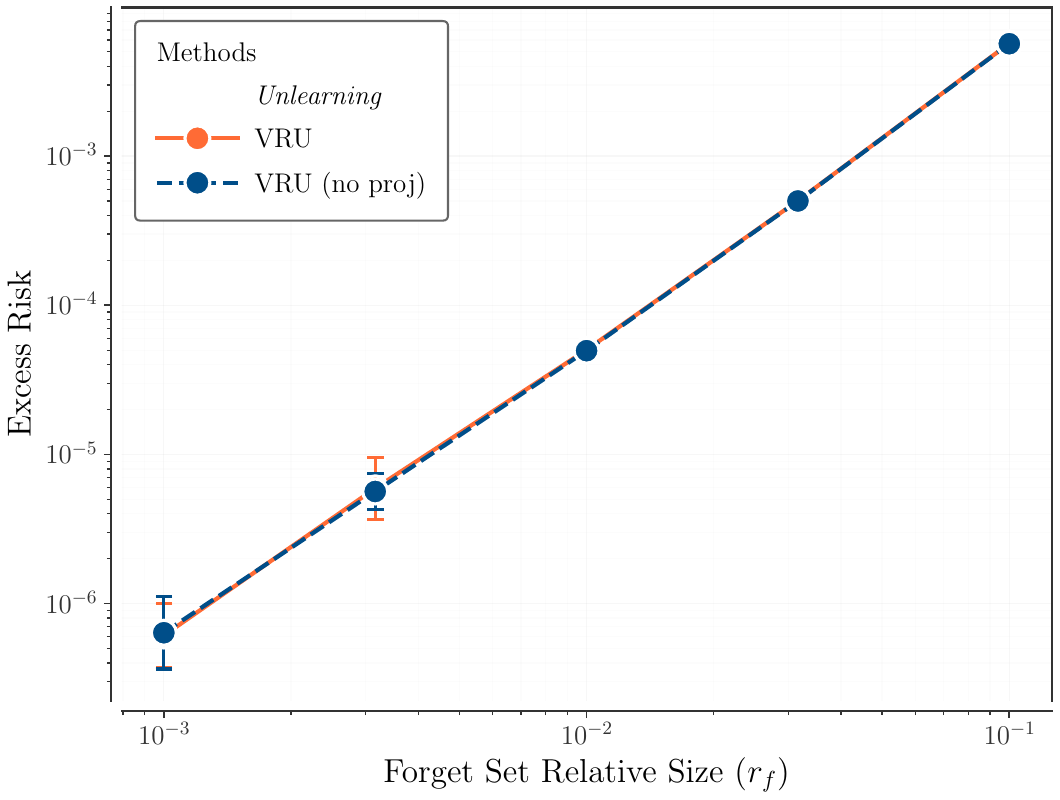}
    \caption{Ablation study on the projection step. Excess risk versus forget fraction $r_f$ for \VRU with and without projection, using $\kappa_{\epsilon,\delta} = 0.1$. The projection step has minimal impact on convergence, indicating \VRU's robustness to this algorithmic choice. Error bars: $\pm 1$ std.\ over 30 runs.}
    \label{fig:app}
\end{figure}

\section{Implementation Details}\label{app:exp_details}

\subsection{Common Experimental Setup}

\paragraph{Dataset and Model.}
All experiments use the Digits dataset~\cite{Digit} with a logistic regression model trained using cross-entropy loss, coupled with a $L2$-weight regularization of weight $0.1$. Training uses a batch size of $8$.

\paragraph{Evaluation Protocol.}
Excess risk measures the gap between the unlearned model's test loss and the retrained-from-scratch baseline. The results are aggregated across seeds, reporting means with standard errors. Geometric means and error bars are used as the measured variables tend to span across several orders of magnitude.

\subsection{Experimental Setup for Figure~\ref{fig:cvx}}

\paragraph{Unlearning Configuration.}
We set the unlearning epoch budget to $T = 10$ and evaluate $n_{r_f} = 5$ uniformly spaced forget ratios in the range $r_f \in [10^{-3}, 0.1]$. We run each experiment with 30 independent random seeds: $0$ through $29$.

\paragraph{Method-Specific Hyperparameters.}
Table~\ref{tab:unlearning_hyperparams_rf} summarizes the hyperparameters for each unlearning method.

\begin{table}[t]
\caption{Hyperparameters for each unlearning and retraining method in \secref{subsec:certified}.}
\label{tab:unlearning_hyperparams_rf}
\centering
\begin{tabular}{lccc}
\toprule
\textbf{Method} & \textbf{Learning Rate} & \textbf{LR Decay} & \textbf{Method Type} \\
\midrule
VRU & 1.1 & 0.55 & Unlearning \\
NFT & $3 \times 10^{-1}$ & 0.8 & Unlearning \\
GD & 2.0 & 0.8 & Retraining \\
SVRG & 1.0 & 0.4 & Retraining \\
SGD & 0.5 & 0.9 & Retraining \\
\bottomrule
\end{tabular}
\end{table}

\subsection{Experimental Setup for Figure~\ref{fig:empirical}}

\paragraph{Unlearning Configuration.}
We set the unlearning epoch budget to $T = 5$, as empirical methods target steeper computational gains,
and evaluate three forget ratios: $r_f \in \{3 \times 10^{-3}, 2 \times 10^{-2}, 10^{-1}\}$. We do not take $r_f$ smaller than $3 \times 10^{-3}$ as results become too unstable when attacking only a few samples, \ie only 0\%, 50\%, or 100\% accuracy for any attack on $\abs{\Df}=1$. We run each experiment with $3$ random seeds $(0, 1, 2)$.

\paragraph{Method-Specific Hyperparameters.}
Table~\ref{tab:unlearning_hyperparams} summarizes the hyperparameters for each unlearning method. The learning rate decays by the specified factor after each epoch. The parameter $\alpha$ denotes the weight of the ascent step in SCRUB and NegGrad+.

\begin{table}[t]
\caption{Hyperparameters for each unlearning method in \secref{subsec:empirical}.}
\label{tab:unlearning_hyperparams}
\centering
\begin{tabular}{lccc}
\toprule
\textbf{Method} & \textbf{Learning Rate} & \textbf{LR Decay} & \textbf{$\alpha$} \\
\midrule
VRU & 1 & 0.6 & -- \\
Fine-Tune & $5 \times 10^{-3}$ & 0.8 & -- \\
NegGrad+ & $3 \times 10^{-3}$ & 0.7 & $5 \times 10^{-3}$ \\
SCRUB & $5 \times 10^{-3}$ & 0.8 & $5 \times 10^{-3}$ \\
\bottomrule
\end{tabular}
\end{table}

As this is an empirical evaluation, unlike for the previous subsection, we do not provide the bounded sensitivity to methods requiring it. 
For \VRU, the noise is applied empirically with $\kdp = 0.1$, which offers a good trade-off, taking $\nu_T=1$, as the smoothness parameter is not known. However, knowing the value of $\nu_T$ would not change our results, as it would simply re-scale the noise, which is equivalent to re-scaling $\kdp$, whose value is not relevant, nor reported in the main text, in this comparison to empirical methods.

\paragraph{Privacy Evaluation via Membership Inference.}
We assess privacy risk using the U-LiRA membership inference attack with 5 shadow models. Target shadow models are trained on $\mathcal{D} = \mathcal{D}_f \cup \mathcal{D}_r$ and then unlearn $\mathcal{D}_f$, while reference shadow models are trained from scratch on $\mathcal{D}_r$ only. To construct the attack set, we sample $n_f = |\mathcal{D}_f|$ elements uniformly at random from the test set to form $\mathcal{D}_{\mathrm{test}}$. The attack set is $\mathcal{D}_{\mathrm{attack}} = \mathcal{D}_f \cup \mathcal{D}_{\mathrm{test}}$, and we report the U-LiRA re-identification attack accuracy as the empirical privacy risk.

\subsection{Impact of averaging and projection step}
The \VRU algorithm, even in its experiment-adjusted form (\Algref{alg:VRU-exp}), requires a projection step on a ball of radius $R^{\text{exp}} \coloneqq\frac{r_f}{1-r_f}\frac{\norm{\nabla^*_f}}{\mu}$ after each iteration. If this ball's radius was to be incomputable, one might wonder if the algorithm would still function properly. We answer this question by analyzing the performance of $VRU$ without the projection step. We place ourself in the experimental setting of \Figref{fig:empirical}, and report the result in \Figref{fig:app}, showing the excess risk as a function of $r_f$, for $5$ logarithmically-spread values of in $[10^{-3}, 0.1]$. We find that the projection step has little to no impact on the loss evolution and is not necessary for the algorithm to effectively reach the optimum in our experimental setting. The reported error bars represent $\pm 1$ standard deviations over $30$ independent runs.

\end{document}